\documentclass{article}

\usepackage{times}
\usepackage{graphicx} 
\usepackage{subfigure} 

\usepackage{natbib}

\usepackage{algorithm}
\usepackage{algpseudocode}

\usepackage[nohyperref,accepted]{icml2016}

\usepackage{mymacros} 
\usepackage{enumitem}
\usepackage{color}
\usepackage{multirow}

\begin{document}

\twocolumn[ \icmltitle{%
  A Tractable Fully Bayesian Method for the Stochastic Block Model}


\icmlauthor{Kohei Hayashi}{hayashi.kohei@gmail.com}
\icmlauthor{Takuya Konishi}{takuya-ko@nii.ac.jp}
\icmladdress{National Institute of Informatics, Tokyo, Japan\\
Kawarabayashi Large Graph Project, ERATO, JST}

\icmlauthor{Tatsuro Kawamoto}{kawamoto@sp.dis.titech.ac.jp}
\icmladdress{Tokyo Institute of Technology, Tokyo, Japan}

\icmlkeywords{boring formatting information, machine learning, ICML}

\vskip 0.3in
]

%
%
%

\begin{abstract}
The stochastic block model (SBM) is a generative model revealing
macroscopic structures in graphs.  Bayesian methods are used for (i)
cluster assignment inference and (ii) model selection for the number
of clusters.  In this paper, we study the behavior of Bayesian
inference in the SBM in the large sample limit.  Combining
variational approximation and Laplace's method, a consistent
criterion of the fully marginalized log-likelihood is
established. Based on that, we derive a tractable algorithm that
solves tasks (i) and (ii) concurrently, obviating the need for an
outer loop to check all model candidates. Our empirical and
theoretical results demonstrate that our method is scalable in
computation, accurate in approximation, and concise in model
selection.
\end{abstract}


\section{Introduction}


Graph clustering has to goals: to detect densely connected subgraphs
and to detect structurally homogeneous subgraphs. While the former
often optimizes an objective function, the latter infers the latent
variables and the parameters of a generative model, for example, the
\emph{stochastic block model (SBM)}.
Despite its simplicity, the SBM is flexible enough to express a range
of structures hidden in real graphs~\citep[Section~2.1]{Leger2013},
and while many variants of the SBM have been proposed, the more
complex models do not always perform better \citep{Peixoto2015}. In
this study, we therefore focus on the most fundamental version of the
SBM.




To uncover the underlying block structures, we need to know the
cluster assignments of the SBM, which can be inferred, in a principled
way, using Bayesian
methods~\citep{Nowicki2001,hastings06,newman07,hofman08,daudin08,mariadassou2010,decelle11,latouche12}.
Incorporating with prior knowledge, Bayesian methods evaluate the
uncertainty of cluster assignments as posterior probabilities.

There are two types of Bayesian method: those that deal with the
uncertainty of both cluster assignments and model parameters and those
that deal with the uncertainty of cluster assignments only. In this
study, we distinguish between them and refer to the former as
\emph{full Bayes} and the latter as \emph{partial Bayes}.
Full Bayes involves intractable integrals and hence approximation
is necessary. Monte Carlo sampling~\citep{Nowicki2001} approximates
these integrals numerically.  Variational Bayesian (VB)
methods~\citep{latouche12} introduce the mean-field approximation and
solve the integrals as an optimization problem.
Despite having less legitimacy, partial Bayes is often favored in
practice because 
of its tractability. \citet{newman07} developed the expectation
maximization (EM) algorithm. \citet{daudin08} introduced variational
EM, which uses the mean-field approximation for posterior
inference. Belief propagation (BP) is an alternative approach for
posterior inference that retains the correlation information among
the cluster assignments and hence makes inference more accurate than the
mean-field approach~\citep{decelle11}.

Bayesian inference can also be used to determine the number of
clusters~\citep{daudin08,decelle11,latouche12}, which we denote by
$K$.  Among all the model candidates $1, \dots, K_{\mathrm{max}}$,
Bayesian theory selects the one that achieves the maximum marginal
likelihood~\citep{schwarz78}. Unfortunately, partial Bayesian methods
are inadequate for this task. Because partial Bayes does not take into
account the uncertainty of the model parameters, it overestimates the model
complexity.
To address this problem, \citet{daudin08} proposed an information
criterion that is, under some conditions, \emph{consistent}, meaning
that it will select the same model as the maximum marginal likelihood
in the large sample limit. Fully Bayesian methods like those proposed
by \citet{Nowicki2001} and \citet{latouche12} have also been
used. These methods, however, share the same problem: scalability. To
obtain the maximum, we need to compute the marginal likelihood for all
model candidates.  This implies that the model selection task is up to
$K_{\mathrm{max}}$ times time-consuming than the cluster assignment
inference task.

Although the SBM has been well analyzed in the \textit{dense} case,
interest has recently turned to \textit{sparse} graphs, in which the
number of edges grows only linearly with the number of nodes. For
example,  a person's Facebook friends do not increase as
the total number of Facebook users increases.  The analysis of 
sparse graphs is more realistic, but is theoretically
challenging because the block structure will be indistinguishable in
the large sample
limit~\cite{reichardt08,decelle11,Krzakala2013,Kawamoto2015Laplacian}.
%
%
Despite its importance, theoretical development of sparse graphs has
been limited compared to their dense counterparts. In particular, no
consistent model selection method for sparse graphs has yet been
established.

In the machine learning community, \emph{factorized asymptotic
  Bayesian (FAB) inference}~\citep{fujimaki12a,hayashi15} has recently been
developed, which approximates fully Bayesian inference for various
latent variable models.
%
%
The FAB method provides both an asymptotic expansion of the marginal
likelihood, termed the \emph{factorized information criterion (FIC)},
and a tractable algorithm to obtain it.  It has a distinctive
regularization effect that eliminates unnecessary model components in
the course of the inference; by initializing the model as $K_{\max}$,
the FAB algorithm converges at some $K\leq K_{\max}$, and 
$K$ can then be used as the selected model.

In this paper, we present an FAB framework for the SBM with the
following appealing features:
\begin{description}[topsep=0pt,itemsep=-1ex,partopsep=1ex,parsep=1ex]
\item[Accurate] Our approximation is consistent for both dense and
  sparse graphs.
\item[Tractable] Our algorithm employs EM-like alternating
  maximization, which is written in closed form.
\item[Scalable] $K$ is automatically selected during posterior
  inference, eliminating the outer loop for $1,\dots,K_{\max}$.
\item[Concise] The selected $K$ is small yet maintaining the same
  prediction accuracy as more complex models.
\item[No hyperparameters] All the parameters are
  estimated.
\end{description}
Our main contributions, which have not been addressed in previous FAB
studies, are as follows.
\begin{itemize}[topsep=0pt,itemsep=-1ex,partopsep=1ex,parsep=1ex]
\item For sparse graphs, the original FAB approximation is invalid
  because of model singularity. To avoid this, we analyze the effects
  of such cases exactly (Section~\ref{sec:asympt-eval-marg}).
\item We evaluate the asymptotic expansion of the joint marginal
  likelihood more rigorously, which improves the error rate
  (Section~\ref{sec:asympt-marg-likel}) and yields interpretable
  regularization terms (Section~\ref{sec:regularizers-effect}).
\item We derive a new BP rule for full Bayes
  (Section~\ref{sec:inference-q}).

\end{itemize}

\paragraph{Notation} Throughout this paper, we denote by $x\approx y$
the relation such that $x=y+O(1/N)$, where $N$ denotes the number of
nodes.



\section{Background}

\subsection{SBM}

Let $\nodes$ and $\edges$ be the sets of $N=|\nodes|$ nodes and
$M=|\edges|$ edges, respectively. The graph $(\nodes,\edges)$ can have
self-edges so that there are ${N+1 \choose 2}=N(N+1)/2$ possible edges.
In the SBM, each node belongs to one of $K$ clusters, and each edge is
assigned to one of $K(K+1)/2$ biclusters. For example, edge $ij$ is
assigned to bicluster $kl$ if node $i$ belongs to cluster $k$ and node
$j$ belongs to cluster $l$. Let us denote by $\X$ the adjacency
matrix, by $\z_i$ the $1$-of-$K$ coding vector representing the
cluster assignment of node $i$, by $\vPi$ the $K\times K$ affinity
matrix that specifies the probability that a pair of nodes to be
connected, and by $\vgamma$ the proportion of cluster assignments
($\sum_k\gamma_k=1$).
Then, the joint log-likelihood of the SBM can be written as
\begin{align}\label{eq:joint}
  &\log p(\X,\Z\mid \vPi, \vgamma, K) 
= \sum_{ik}z_{ik} \log \gamma_k\notag
\\&+\sum_{i\leq j}\sum_{kl}z_{ik}z_{jl}(\log \pi_{kl}^{x_{ij}} 
+ \log(1 - \pi_{kl})^{1-x_{ij}}) .
\end{align}
For brevity, we omit $K$ from the notation when it is obvious from the
context.


\subsection{EM Algorithm}\label{sec:em-algorithm}

By following a Bayesian manner, we marginalize $\Z$ out from the
likelihood. The naive marginalization requires all combinations of
$\Z$ to be computed, which is computationally infeasible. Instead, we
consider its variational form,
\begin{align}\label{eq:partial-variational}
  &\log p(\X| \vPi, \vgamma) 
  = \E_{\tilde q}[\log p(\X,\Z| \vPi, \vgamma)] + H(\tilde q) + \KL{\tilde q}{\tilde p},
\end{align}
where $\tilde q$ is any distribution over $\Z$, $H( q)=-\E_{ q}[\log
q(\Z)]$ is the entropy, $\KL{q}{ p}=\E_q[\log q(\Z)/p(\Z)]$ is the KL
divergence, and
\begin{align}\label{eq:posterior}
  \tilde p(\Z) =& p(\Z\mid\X,\vPi, \vgamma) =  \frac{p(\X,\Z\mid\vPi, \vgamma)}{\sum_{\Z}p(\X,\Z\mid\vPi, \vgamma)}
\end{align}
is the posterior.

The EM algorithm can be used to obtain the posterior $\tilde p(\Z)$
and the maximum likelihood estimators by iterating two steps called
the E-step and the M-step~\citep{newman07}. Let
\begin{align*}
  \bar\z&=\frac{1}{N}\sum_i \z_i,
  \quad
  \overline{\z\z^\T}
  =\frac{1}{N^2}(\sum_{ij} x_{ij}\z_i\z_j^\T
  +
  \diag(\sum_i x_{ii}\z_{i}))
\end{align*}
be the sufficient statistics. Here, $\bar z_k$ represents the
empirical proportion of nodes assigned to cluster $k$ and
$\overline{zz}_{kl}$, the $kl$-th element of $\overline{\z\z^\T}$,
represents the empirical average of edges assigned to
bicluster $kl$.
%
In the E-step, we update $\tilde q$ by minimizing the KL divergence with
the old estimators of $\vPi$ and $\vgamma$. Then, in the M-step, we maximize
$\E_{\tilde q}[\log p(\X,\Z| \vPi, \vgamma)]$ with respect to $\vPi$
and $\vgamma$, which are obtained in closed form.
\begin{proposition}\label{prop:joint-maximum}
  $\E_{\tilde q}[\log p(\X,\Z| \vPi, \vgamma)]$ has a unique maximum
  at $\vgamma=\ML{\vgamma}(\tilde q)\equiv\E_{\tilde
    q}[\bar\z]$. Also, for $\{(k,l)\mid\E_{\tilde q}[\bar
  z_k]\E_{\tilde q}[\bar z_l]>0\}$, $\E_{\tilde q}[\log p(\X,\Z| \vPi,
  \vgamma)]$ has a unique maximum at $\pi_{kl}=\ML{\pi}_{kl}(\tilde
  p)$ where, by denoting $\div$ the element-wise division,
  \begin{align*}
    \ML{\vPi}( q)\equiv
    \E_{ q}[\ZZ]
    \div
    \left(
      \E_{ q}[\bar \z]\E_{ q}[\bar \z]^\T + \frac{1}{N}\diag(\E_{ q}[\bar\z])
    \right).
  \end{align*}
\end{proposition}
%

\subsection{BP}\label{sec:belief-propagation}

The E-step requires $\tilde p(\Z)$ to be computed, but its normalizing
constant is computationally infeasible. One solution is to restrict
the class of $q(\Z)$.  For example, \citet{latouche12} proposed a
variational EM approach that approximates $q(\Z)$ from the mean-field
expression $q(\Z)=\prod_iq(\z_i)$.  However, because $\{\z_i\}$ are
mutually dependent in the true posterior, this may cause
a huge approximation error.


BP is an alternative approach to obtaining $\tilde
p(\Z)$~\citep{decelle11}. BP aggregates local marginal information as
``message'' and computes marginalization efficiently by exploiting the
graphical structure of a probabilistic model.
%
%
For $(i,j)\in\edges$, the message is given as
\begin{align}\label{eq:message}
  \tilde\vmu^{j\to i} \propto \exp(\log\vgamma + \a_j +\sum_{s\in\nodes_j\quot i}\log\vPi\tilde\vmu^{s\to j})
\end{align}
where $\nodes_j=\{s|(s, j)\in\edges\}$ is the set of the neighbors of
node $j$ and
$a_{jk}=\sum_{s\notin\nodes_j}\log(\1-\vPi \tilde\vmu^{s\to j})_{k}$
is the log-factor of the unconnected nodes. The sum-product rule then
gives the marginal expectations as 
\begin{subequations}\label{eq:moments}
  \begin{align}
    \E[\z_j]&\propto\tilde\vmu^{j\to i} * \vPi\tilde\vmu^{i\to j},
    \\
    \E[\z_{i}\z_{j}\mid x_{ij}=1]&\propto \vPi * \tilde\vmu^{j\to
      i}(\tilde\vmu^{i\to j})^\T
  \end{align}
\end{subequations}
where $*$ denotes the Hadamard product.
Note that the graphical model of the SBM has many loops.  Thus, BP on
the SBM does not converge to the exact posterior.  Nevertheless, in
many cases, BP gives a better inference than variational approaches
using the mean-field approximation~\citep{decelle11}.

\subsection{Inference on a Sparse Graph}\label{sec:infer-sparse-graph}

When a graph is dense, the inference of $\Z$ is relatively easy.  We
say a graph is dense if there exists a constant $a$ such that
$a<\pi_{kl}<1-a$ for all $k$ and $l$, meaning that $M=\Theta(N^2)$.
\citet{celisse12} show that, if a graph is dense and assuming some
minor conditions, $\tilde p(\Z)$ converges almost surely to the
indicator of true cluster assignments for $N\to\infty$. Therefore, the
uncertainty of the posterior of $\Z$ decreases as $N$ increases, i.e.,
the posterior becomes as a point estimator at the large sample limit.

In contrast, the inference problem becomes more difficult in a sparse
graph~\citep{reichardt08,decelle11}. We say a graph is sparse when
$\pi_{kl}=\Theta(1/N)$ for all $k$ and $l$. In this case, $\pi_{kl}$
approaches zero as $N$ increases, and the uncertainty of $\Z$ remains
even as $N\to\infty$. Accurate inference of the posterior is thus more
important than the case of dense graphs, which motivates the use of
BP.

Sparseness also confers a computational advantage on BP.  For a dense
graph, the updating of all the BP messages requires
$O(N^3K^2)$---there are $O(N^2)$ messages for each node, each message
requires $O(K^2)$, and all nodes must be updated in a single sweep.
To reduce the computational burden, \citet{decelle11} proposed an
efficient approximation of $\a_j$ for a sparse graph as, by using the
fact that $\tilde\vmu^{s\to j} \approx \E[\z_s]$,
\begin{align}\label{eq:approx-a}
\a_j \approx - \sum_{s\notin\nodes_j}\vPi \tilde\vmu^{s\to j} 
\approx - \sum_{s \in \nodes}\vPi \, \E[\z_s]
\equiv \tilde\a. 
\end{align}
The vector $\tilde\a$ is node-independent, allowing the computation of
unconnected nodes in \eqref{eq:message} to be omitted.
%
In this approach, the messages from unconnected nodes are replaced by
an external field.  Therefore, in sparse graphs, the complexity is 
reduced to $O(MK^2)$, because there are $M$ edges
and 
$O(1)$ neighbors for each node.


%


\section{Asymptotic Evaluation of Marginals}\label{sec:asympt-eval-marg}

Hereafter, for mathematical convenience, we employ the
exponential-family representation of the SBM\eqref{eq:joint}:
\begin{align}\label{eq:joint-exp}
\log p(\X,\Z|\vTheta,\veta) 
=& \sum_{i\leq j}\sum_{kl}z_{ik}z_{jl}(x_{ij}\theta_{kl} -\psi(\theta_{kl})) \notag\\
&  + \sum_{i}(\sum_{k<K} z_{ik}\eta_{k} - \phi(\veta)), 
\end{align}
where $\veta\in(-\infty, \infty)^{K-1}$ is the natural parameter of
$\vgamma$ %
and $\phi(\veta)=\log(1+\sum_{k<
  K}\mathrm{e}^{\eta_k})$ is the cumulant generating function of the
multinomial distribution.  Similarly, $\theta_{kl}\in(-\infty,\infty)$
is the natural parameter of $\pi_{kl}$ and
$\psi(x)=\log(1+\exp(x))$ 
is the cumulant generating function of the Bernoulli distribution.

Note that, while the parametrization is different, both
\eqref{eq:joint} and \eqref{eq:joint-exp} represent the same
probabilistic model.
Indeed, there is a one-to-one mapping from one parametrization to the
other.  For example, the derivative of the cumulant generating
function is the mapping from the natural parameter to the original
parameter, e.g., $\psi'(\theta_{kl})=\pi_{kl}$ where $\psi'(\cdot)$ is
the sigmoid function. Also, $\phi'(\cdot)$ is the softmax function.
%

\subsection{Asymptotic Joint Marginal}

Our main goal is to obtain the fully marginalized log-likelihood. Using
the exponential-family representation, this is written as
\begin{align}\label{eq:full-marginal}
  \log p(\X) = \log \sum_{\Z}\int p(\X,\Z|\vTheta,\veta) p(\vTheta)p(\veta) \d\vTheta\d\veta 
\end{align}
where $p(\vTheta)$ and $p(\veta)$ are the prior distributions of the
parameters. The marginalization with respect to $\vTheta$ and $\veta$
has no analytical solution in general. Also, the computational
infeasibility of $\Z$ discussed in Section~\ref{sec:em-algorithm} still
remains.
%
We first resolve this issue of infeasibility by using the variational
form. As analogous to \eqref{eq:partial-variational}, the full
marginal~\eqref{eq:full-marginal} is rewritten as
\begin{align}\label{eq:full-variational}
  \log p(\X) 
  = \E_q[\log p(\X,\Z)] + H(q) + \KL{q}{p^*}
\end{align}
where $p^*(\Z)\equiv p(\Z|\X)$ is the marginalized posterior in
which, in contrast to $\tilde p$, the parameters are marginalized out.

In \eqref{eq:full-variational}, the joint marginal 
$$p(\X,\Z)=\int
p(\X,\Z|\vTheta,\veta)p(\vTheta)p(\veta)\d\vTheta\d\veta$$ still
contains the intractable integrals with respect to $\vTheta$ and
$\veta$. However, the joint marginal is more manageable than the full
marginal \eqref{eq:full-marginal}. In the joint marginal, \emph{$\Z$
  is not latent but rather is regarded as given.} That is, when
evaluating $p(\X,\Z)$, we can focus on a specific cluster assignments
determined by $\Z$, i.e., the uncertainty of $\Z$ is completely
excluded. In addition, as shown in
Proposition~\ref{prop:joint-maximum}, $p(\X,\Z|\vTheta,\veta)$ has a
unique maximum %
if there is no empty cluster (i.e., $\forall_k \bar z_k>0$.)  In this
situation, $p(\X,\Z|\vTheta,\veta)$ has a single peak and its main
contribution to the integral is made by the neighbor of the peak; the
contribution of the other part diminishes asymptotically. For this
type of integral, \emph{Laplace's method} gives a very accurate
approximation.
\begin{lemma}[Laplace's method~\citep{wong01}]\label{thm:laplace}
  Let $f:\mathcal{X}\to\R$ and $g:\mathcal{X}\to\R$ be infinitely
  differentiable functions on $\mathcal{X}\subseteq \R^D$.  Suppose
  the integral $I\equiv\log\int_{\mathcal{X}} \exp(-Nf(\x)) g(\x)\d\x$
  converges absolutely for sufficiently large $N$.  If $f$ has a
  unique maximum at $\hat\x$ that is an interior point of
  $\mathcal{X}$ and the Hessian matrix $\nabla\nabla f(\ML{\x})$ is
  positive definite, then
  \begin{align*}
    I\approx
    -Nf(\ML{\x}) + \log g(\ML{\x}) - \frac{1}{2}\log|\nabla\nabla f(\ML{\x})| 
    - \frac{D}{2}\log\frac{N}{2\pi}.
  \end{align*}
\end{lemma}

Letting $Nf(\x)=-\log p(\X,\Z|\vTheta,\veta)$ and
$g(\x)=p(\vTheta)p(\veta)$ with $\x=\{\vTheta,\veta\}$, the joint
marginal is approximated by Lemma~\ref{thm:laplace}.
Before the approximation, however, we have to check the conditions of
Laplace's method carefully, especially about 1) the regularity of the
Hessian matrix and 2) the interiority of the maximum. Although these
conditions are satisfied for most instances of $\Z$, they are
sometimes violated. For example, as
Proposition~\ref{prop:joint-maximum} suggests, if cluster $k$ is empty
(i.e., $\bar z_k=0$,) the joint likelihood takes the same value with
any $\{\theta_{kl}|k\leq l\leq K\}$ and $\{\theta_{lk}|1\leq l< k\}$,
i.e., the Hessian matrix becomes singular. Moreover, if no edge
belongs to bicluster $kl$ ($\bar \zz_{kl}=0$,) the maximum occurs at
$\theta_{kl}\to-\infty$, which is an endpoint and condition 2) is
violated.
In particular, the case of $\theta_{kl}\to-\infty$ is equivalent to
the case of $\pi_{kl}\to 0$ and thus is critical for sparse graphs.

For the violated cases, we evaluate the integral exactly. Combining
this with the result of Laplace's method, we obtain an asymptotic
expansion of $\ln p(\X,\Z)$, which is the main contribution of this
paper. The proof is shown in Appendix.
\begin{theorem}\label{lem:joint}
  Suppose $K=O(1)$ and $p(\vTheta)p(\veta)$ is infinitely
  differentiable. Given $\Z$, let $\S=\{k|\bar z_k>0\}$ be the set of
  the non-empty clusters and $\mathcal{S}'=\mathcal{S}\quot K$; let
  $M_*= N^2\min_{k\in\S}\bar z_{k}^2$ be the minimum size of the
  non-empty clusters.  We define the indicator function for non-empty
  clusters as $\delta_k=\mathbb{I}(\bar z_k>0)$ and denote by
  $K_{\bar\z}=\sum_k\delta_k$ the number of non-empty clusters. We use
  a similar notation for non-empty biclusters as
  $\Delta_{kl}=\ind(\zz_{kl}>0)$ and $K_{\zz}=\sum_{k\leq
    l}\delta_k\delta_l\Delta_{kl}$. Then, we have
\begin{align}\label{eq:asymptotic-joint}
  \log p(\X,\Z) =& \joint(\Z) + C + O(\min(N,M_*)^{-1}),
\end{align}
\begin{align*}
\joint(\Z) \equiv&
\log p(\X,\Z|\ML{\vTheta},\ML{\veta}) 
-R_1(\bar\z)
 -R_2(\bar \z, \overline{\z\z^\T}) - \ell_N,
\\
R_1(\bar\z) 
\equiv& \frac{1}{2}\sum_{k}\delta_k\log\bar z_k, 
\\
R_2(\bar \z, \overline{\z\z^\T}) 
\equiv& \frac{1}{2}\sum_{k\leq l}\delta_k\delta_l\Delta_{kl}\log \zz_{kl}(1-\ML{\pi}_{kl}),
\\
\ell_N\equiv& \frac{K_{\bar\z}-1}{2}\log N + \frac{K_{\zz}}{4}\log \frac{N(N+1)}{2},
\\
C\equiv& \log p(\ML{\vTheta}_{\S}) + \log p(\ML{\veta}_{\S})
\\
 &+\sum_{k\leq l}\delta_k\delta_l(1-\Delta_{kl})P_{kl} + Q_{\S'} + \const,
\\
P_{kl}\equiv& \log \int \left(\frac{1}{1+\e^{\theta_{kl}}}\right)^{\bar
      M_{kl}}p(\vTheta_{\quot \S}|\ML{\theta}_{kl})\d\vTheta_{\quot
      \S},
\\
Q_{\S}\equiv& \log\int \left(\frac{1}{1+\sum_{l\notin\S}\e^{\eta_l -
          \log\ML{\alpha}}}\right)^N p(\veta_{\quot
      \S}|\ML{\veta}_{\S})\d\veta_{\quot \S},
\end{align*}
where $\bar M_{kl} = \frac{N^2}{2}\bar z_k (\bar z_l +
\frac{\ind(k=l)}{N})$ and $\ML{\alpha} =
1+\sum_{k\in\S}\e^{\ML{\eta}_k}$.
\end{theorem}
%
%
The result of Theorem~\ref{lem:joint} is fairly intuitive and
interpretable. Marginalization over non-empty clusters $\{k|\bar
z_k>0\}$ and biclusters $\{(k,l)|\zz_{kl}>0\}$ provides a BIC-like
``maximum likelihood + penalty'' term as $\joint(\Z)$. Since $\log
p(\X,\Z|\ML{\vTheta},\ML{\veta})$ is the maximum likelihood, it
monotonically increases as $K$ increases. In contrast, the value of
$\ell_N$ decreases on the order of $\log N$ as the number of non-empty
clusters increases, which penalizes model complexity. $R_1$ and $R_2$,
resulting from the Hessian matrix, represent additional model
complexity, where BIC does not have such term. These effects are
discussed in detail in Section~\ref{sec:regularizers-effect}.

The contribution of empty (bi)clusters is separated from the main term
and appear as $P_{kl}$ and $Q_{\S}$, which place an extra penalty on
model redundancy. The integrand of $P_{kl}$ is the $\bar M_{kl}$-th
power of the sigmoid function and the prior density, where $\bar
M_{kl}$ roughly represents the proportion of bicluster $kl$. Because
the $\bar M_{kl}$-th power of the sigmoid function has a change
point at $\theta=-\log \bar M_{kl}$, it can be approximated by a
step function where the step point is $-\log \bar M_{kl}$. This
approximates $P_{kl}$ as the log cumulative distribution of the prior:
$P_{kl}\simeq \log\Pr(\theta_{kl}\sim
p(\theta_{kl}|\ML{\vTheta}_{\quot \S})<-\log \bar M_{kl})$. Because
the logarithm of a cumulative distribution is non-positive, it
decreases the likelihood depending on the choice of the priors. A
similar observation holds for $Q_{\S}$.


%

\subsection{Asymptotic Marginal Likelihood}\label{sec:asympt-marg-likel}

By substituting \eqref{eq:asymptotic-joint} into
\eqref{eq:full-variational} and setting $q=p^*$, we obtain the
approximation of $\log p(\X)$, which we refer to as the
\emph{fully factorized information criterion}:
\begin{align}\label{eq:F2IC}
  \FFIC(K) =& \E_{p^*}[\log p(\X,\Z\mid\hat\vTheta,\hat\veta,K) - R_1(\bar\z) \notag
\\
  & - R_2(\bar \z,\overline{\z\z^\T})  - \ell_N + C] + H(p^*).
\end{align}
\begin{corollary}\label{cor:FIC}
  Under the assumptions of Theorem~\ref{lem:joint}, 
  \begin{align*}
    \log p(\X|K) = \FFIC(K) +O(1).
  \end{align*}
  In addition, if \assum{low-empty-prob-weak} $p^*$ satisfies $\Pr(\bar z_k\not=0\cap\bar z_k=o(N^{-{1\over
      2}}))=O(N^{-1})$ for all $k$,
  \begin{align*}
    \log p(\X|K) \approx \FFIC(K).
  \end{align*}
\end{corollary}
Corollary~\ref{cor:FIC} shows that $\FFIC$ is consistent with the
marginal log-likelihood.
In addition, if \ref{asm:low-empty-prob-weak} is satisfied, i.e., if
the almost-empty clusters having $o(\sqrt{N})$ nodes are rarely
generated by the marginal posterior, the approximation becomes more
accurate and the error decreases as $O(N^{-1})$.

%

\section{Posterior Inference and Model Selection}\label{sec:asympt-full-bayes}

\subsection{Lower Bound of $\FFIC$}\label{sec:lowerbound-ffic}

Computing $\FFIC$ is a nontrivial task due to four challenges: 
\begin{itemize}
\item[1)] evaluation of $\{P_{kl}\}$ and $Q_{\S}$,
\item[2)] estimation of
$\MJL{\vTheta}$ and $\MJL{\veta}$,
\item[3)] inference of $p^*(\Z)$, and
\item[4)] computation of $\E_{p^*}[\cdot]$ in $R_1,R_2$, and $\ell_n$.
\end{itemize}

To avoid 1), we employ the (improper) uniform priors for $\vTheta$ and
$\veta$. If $N<\infty$ and $|\ML{\eta}_{k}|<\infty$ for all non-empty
clusters, then $\{P_{kl}\}$ and $Q_{\S}$ with the uniform priors lose their
dependency on $N$ and become $P_{kl}=\log\frac{1}{2}$ and
$Q_{\S}=|\S|\log\frac{1}{2}$. Also, $\log p(\ML{\vTheta})$ and $\log
p(\ML{\veta})$ become constants. We therefore ignore $C$ in
\eqref{eq:asymptotic-joint} as a constant.

Difficulties 2)--4) are bypassed them by constructing a tractable
lower bound of $\FFIC$.

For 2), because the average of maxima is greater than or equal to the maximum
of the average,
\begin{align}
  \E_q[\log p(\X,\Z\mid\hat\vTheta,\hat\veta)]\geq
  \E_q[\log p(\X,\Z\mid\EML\vTheta,\EML\veta)]
\end{align}
holds for any $q(\Z)$, where 
$$\{\EML\vTheta,\EML\veta\}=\argmax_{\vTheta,\veta}\E_q[\log
p(\X,\Z|\vTheta,\veta)]$$ 
is the global maximizers having closed-form
solutions:
\begin{align}\label{eq:EML}
  \EML\theta_{kl} =& (\psi')^{-1}(\ML{\pi}_{kl}(q)),
\qquad
  \EML\eta_k = (\phi')^{-1}(\ML{\gamma}_k(q)).
\end{align}
$(\psi')^{-1}$ is the logit function and $(\phi')^{-1}$ is the
inverse softmax function. 

For 3), to obtain $p^*(\Z)$, we again use
Theorem~\ref{lem:joint}. Because $p^*(\Z)=p(\Z|\X)\propto p(\X,\Z)$,
collecting the $\Z$-related terms in \eqref{eq:asymptotic-joint} gives
$p^*(\Z)= q^*(\Z)(1+O(\min(N,M_*)^{-1}))$ where
\begin{align} \label{eq:marginal-posterior}
  q^*(\Z) &\propto p(\X,\Z|\ML{\vTheta},\ML{\veta})
  \e^{-R_1(\bar\z)-R_2(\bar
    \z,\overline{\z\z^\T})-\ell_N+C}.
\end{align}
We then use $q^*$ instead of $p^*$. Note that because of the
nonnegativity of the KL divergence, $\FFIC(p^*)\geq \FFIC(q)$ holds
for any $q(\Z)$, and using $q^*$ gives a lower bound. 

For 4), we obtain lower bounds using Jensen's inequality. For $R_1$,
we use a lower bound of $-\E[\delta_k\log\bar z_k]\geq -\E[\log(\bar
z_k + \frac{1}{N})]\geq -\log(\E\bar z_k + \frac{1}{N})$.
For $R_2$, because $\ML{\pi}_{kl}=\Theta(N^{-1})$ for a sparse
graph\footnote{Constructing a lower bound for a dense graph is also
  possible.}, $\log
\zz_{kl}(1-\ML{\pi}_{kl})=\log\zz_{kl}+\Theta(N^{-1})$ and the effect
of $(1-\ML{\pi}_{kl})$ is negligible. Also,
$-\E[\delta_k\delta_l\Delta_{kl}\log\zz_{kl}]\geq-\log(\E\zz_{kl}+\frac{1}{N^2})$. 
A similar lower bound holds for $\ell_N$.

By combining these, we obtain the lower bound of $\FFIC$ as
\begin{align}\label{eq:FFIC-lowerbound}
&\E_{q}[\log p(\X,\Z|\EML\vTheta,\EML\veta)] - \tilde R_1(\E_{q}\bar\z) 
- \tilde R_2(\E_{q}\overline{\z\z^\T})\notag
\\
&- \tilde\ell_N + H(q),
\end{align}
where $\tilde R_1(\bar\z)=\frac{1}{2}\sum_{k}\log (\bar z_k+
\frac{1}{N})$, $\tilde R_2(\overline{\z\z^\T})=\frac{1}{2}\sum_{k\leq
  l}\log (\zz_{kl} + \frac{1}{N^2})$, and
$\tilde\ell_N=\frac{K-1}{2}\log N + \frac{K(K+1)}{4}\log
\frac{N(N+1)}{2}$.

\subsection{Inference of $q(\Z)$}\label{sec:inference-q}

Similarly to the EM algorithm, we need to optimize $q$ in
\eqref{eq:FFIC-lowerbound} that tightens the lower bound. For that
purpose, we derive a new BP rule. 

Substituting the above approximations $(R_1\simeq\tilde R_1,
R_2\simeq\tilde R_2,\ell_N\simeq\tilde\ell_N)$ to
\eqref{eq:marginal-posterior}, the sum-product rule defines a message
for $(i,j)\in\edges$ as
\begin{align}
  \vmu^{j\to i} \propto &\exp(\log\EML\vgamma  + \a_j+\sum_{s\in\nodes_j\quot i}\log\EML\vPi\vmu^{s\to j} - \vlambda_j),\label{eq:marginal-message}\\
  \lambda_{jk}=&\log\E_{\quot j}[\exp({\tilde R_1(\bar\z)+\tilde R_2(\overline{\z\z^\T})})|z_{jk}=1],
\label{eq:expected-R}
\end{align}
where $\E_{\quot j}[f(\Z)]=\sum_{\Z_{\quot
    j}}f(\Z)\prod_{s\not=j}\mu^{s\to j}(\z_j)$ denotes the expectation
by the joint message except node $j$.

Because the log-expectation-exponential in $\vlambda_j$ is
intractable, we need to approximate it. The key fact is that each
message is normalized, and $\prod_{s\not=j}\mu^{s\to j}(\z_j)$ can be
seen as the probabilities of $\{\z_s\}_{s\not= j}$. By using this, we
obtain that $\E_{\quot j}[\exp(\tilde R_1(\bar\z))|z_{jk}=1] \simeq
\exp(\tilde R_1(\E_{\quot j}[\bar\z|z_{jk}=1]))$, which is written as
\begin{align*}
  &
  \frac{1}{2}\log(\E_{\quot j}[\bar z_k|z_{jk}=1]+\frac{1}{N})
  +\frac{1}{2}\sum_{l\not=k}\log(\E_{\quot j}[\bar
  z_l|z_{jl}=0]+\frac{1}{N})
  \\
  &= \frac{1}{2} \log\frac{N\E_{\quot j}[\bar z_k|z_{jk}=1] +
    1 }{N\E_{\quot j}[\bar z_k|z_{jk}=0]+ 1} + b.
\end{align*}
Note that $b=\frac{1}{2}\sum_{l=1}^K\log(\E_{\quot j}[\bar z_l|z_{jl}=0]+\frac{1}{N})$ does not depend on $k$ so that we ignore it
as a constant.
Also, in a sparse graph, recall that $\mu^{s\to j}(\z_s) \approx
\E_q[\z_s]$ for $j\notin\nodes_s$
(Section~\ref{sec:infer-sparse-graph}). Therefore,
\begin{align*}
  \E_{\quot j}[\bar \z]
  =& \frac{\E_{\quot j}[\sum_{s\not=j}\z_s] + \z_j}{N}
  \approx \frac{N\EML\vgamma - \E_q[\z_j] + \z_j}{N}.
\end{align*}
Similar approximation can be used for $\E_{\quot
  j}[R_2(\bar\z,\overline{\z\z^\T})]$.  

Back-substituting these into \eqref{eq:expected-R}, we obtain an
approximate BP message as
\begin{align}\label{eq:marginal-approx-message}
  \vmu^{j\to i} \propto &\exp(\log\EML\vgamma  + \tilde\a_j+\sum_{s\in\nodes_j\quot i}\log\EML\vPi\vmu^{s\to j} - \tilde\vlambda_j).
\end{align}
Here, $\tilde\vlambda_j$ corresponds to the $\FFIC$ penalty terms
defined as
\begin{subequations}
\begin{align}
\tilde\lambda_{jk}=
&\frac{1}{2}\log\frac{[\t_{\quot j}]_k + 1}{[\t_{\quot j}]_k}\label{eq:lbp-r1}
\\
&+\frac{1}{2}\sum_l
  \log\frac{[\vT_{\quot j}]_{kl}+\sum_{s\in\nodes_j}\E_q[z_{sl}]}{[\vT_{\quot j}]_{kl}},
\label{eq:lbp-r2}
\end{align}
\end{subequations}
where $\t_{\quot j}=N\EML\vgamma_k - \E_q[\z_{j}]+\1$ and $\vT_{\quot
  j}=N^2\EML\vgamma\EML\vgamma^\T -
\E_q[\z_{j}]\E_q[\sum_{s\in\nodes_j}\z_{s}]^\T+\1\1^\T$. These
expectations can be computed in the same way as \eqref{eq:moments}.
%

We refer to the inference algorithm using this messages as
\textit{FABBP} (Algorithm~\ref{alg:lbp}). Thanks to the approximation
for sparse graphs, the time complexity of FABBP is $O(MK^2)$, as in
the original BP. In accordance with $\FFIC$, we refer to the
alternating update of $q(\Z)$ and $\{\EML\vTheta,\EML\veta\}$ as the
\emph{F$^2$AB algorithm} (Algorithm~\ref{alg:f2ab}).
\begin{algorithm}[tb]
\caption{\texttt{FABBP}($\vgamma$,$\vPi$)}\label{alg:lbp}
\begin{algorithmic}
  \Repeat
    \For{ randomly choosing $(i,j)\in\edges$}
      \State Update $\E\z_i$ and 
      $\vmu^{i\to j}$ by \eqref{eq:marginal-approx-message}
      \State $\vdelta^i =\E\z_i^{\new} - \E\z_i^{\old}$,~~$\vdelta^{i\to j} =\vmu^{i\to j}_{\new} - \vmu^{i\to j}_{\old}$
      \State $\h\gets \h + \vdelta^i$  
      \State $\E\bar\z \gets\E\bar\z +  \vdelta^i/N$
      \State $\E\overline{\z\z^\T} \gets \E\overline{\z\z^\T} + \vPi * \vdelta^{i\to j}(\vmu^{j\to i})^{\T}/N^2$
      \If{$\E\bar z_k<0.1/N$ for $k=1,\dots,K$}
      \State Remove cluster $k$ and $K\gets K - 1$
      \EndIf
    \EndFor
  \Until{$\sum_{(i,j)\in\edges}|\vdelta^{i\to j}|/M < 10^{-2}$ }
  \State\Return $q=\{\E\bar\z,\E\overline{\z\z^\T}\}$
\end{algorithmic}
\end{algorithm}
\begin{algorithm}[tb]
\caption{The F$^2$AB algorithm of the SBM.}\label{alg:f2ab}
\begin{algorithmic}
\State Initialize $K=K_{\max}$ and $\vmu^{i\to j}$ for $(i,j)\in\edges$ randomly
\State Initialize $\vPi$ by the spectral method~\cite{rohe11}
\Repeat
  \State $q\gets$ \texttt{FABBP}($\vgamma,\vPi$)
\State $\vgamma\gets\ML{\vgamma}(q)$~~and~~$\vPi\gets\ML{\vPi}(q)$
\Until{$\max_{kl}|\pi_{kl}^{\old}-\pi_{kl}^{\new}| < 10^{-8}$}
\end{algorithmic}
\end{algorithm}

\subsection{Penalization Effect of $R_1$ and $R_2$}\label{sec:regularizers-effect}

In $\FFIC$~\eqref{eq:F2IC}, the marginalization with respect to
$\vTheta$ and $\veta$ induces additional terms $R_1$ and $R_2$ via
Laplace's method. Their effects are inherited in FABBP as
$\tilde\vlambda$, which does not exist in the original BP message
\eqref{eq:message}. In fact, $\tilde\vlambda$ diminishes the size of
redundant clusters. For example, consider the effect of $R_1$, which
appears as \eqref{eq:lbp-r1}. When $\EML\vgamma_k=\Omega(1)$,
$\t_{\quot j}\approx N\EML\vgamma_k$. This simplifies
\eqref{eq:lbp-r1} to
\begin{align}\label{eq:simple-r1}
  \frac{1}{2}\log\frac{\t_{\quot j}+1}{[\t_{\quot j}]_k}
  \simeq 
  \frac{1}{2}\log\left(1 + \frac{1}{N\EML\gamma_k}\right).
\end{align}
This suggests that, if $\Omega(N)$ nodes are assigned to cluster $k$,
$1/N\EML\gamma_k\to 0$ and \eqref{eq:simple-r1} goes to zero, i.e.,
$R_1$ penalizes nothing. In contrast, if cluster $k$ has only a few
nodes, $1/N\EML\gamma_k$ remains a constant, and $R_1$ reduces the
message proportion of cluster $k$ (Figure~\ref{fig:R1}).
\begin{figure}[tb]
  \centering
  \includegraphics[width=1\linewidth]{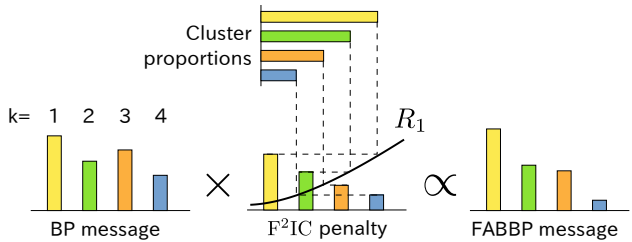}
  \caption{Penalization effect of $R_1$.}
  \label{fig:R1}
\end{figure}

Remarkably, $R_2$ has a different penalization effect that complements
that of $R_1$. As in \eqref{eq:simple-r1}, \eqref{eq:lbp-r2} can be
approximated as
\begin{align}
\frac{1}{2}\sum_l
  \log\left(1 + \frac{\sum_{s\in\nodes_j}\E_q[z_{sl}]}{N^2\EML\gamma_k\EML\gamma_l}\right).
\end{align}
Unlike the case of $R_1$, two cluster sizes $\EML\gamma_k$ and
$\EML\gamma_l$ appear in the denominator together. Because the product
$\EML\gamma_k\EML\gamma_l$ represents the proportion of bicluster
$kl$, $R_2$ penalizes each cluster if it has many small
(low-proportional) biclusters. Thus, $R_2$ evaluates the redundancy of
clusters in a more fine-grained way than $R_1$ does---the $R_1$
penalty depends on cluster proportions, whereas the $R_2$ penalty
depends on bicluster proportions.

These penalization affect all the BP messages, and redundant clusters
disappear in the FABBP iterations.  For this reason, it is not
necessary to compute the $\FFIC$ lower bound for model selection; if
the initial model $K_{\max}$ is sufficiently large, the FABBP
algorithm will automatically determine an adequate $K$ .


\section{Related Work}\label{sec:related-work}

\paragraph{Bayesian Methods}

\citet{Nowicki2001} employed a Monte Carlo method for Bayesian
inference. Although it is accurate, their method cannot handle graphs
larger than a few hundred nodes. To deal with large graphs, the VB
method using the EM algorithm is often
used~\citep{newman07,hofman08,daudin08,latouche12}.  One of the
standard approaches is to update the latent variables and
model parameters iteratively using the uninformative
priors~\citep{hofman08,latouche12,mariadassou2010}.
An alternative approach is to use BP for the cluster assignment
inference~\citep{hastings06,decelle11}.

%


Bayesian nonparametric methods provide an alternative way of
determining $K$~\citep{antoniak74,griffiths11}.  \citet{kemp06}
proposed the infinite relational model (IRM), which extends the SBM to
handle an infinite number of clusters.  In a way similar to FAB, $K$
is automatically determined during the learning process.  However,
\citet{miller13} proved that the Dirichlet process mixtures
(DPMs)---the Bayesian nonparametric extension of mixture
models---overestimate $K$.  Because the IRM is
closely related to the DPM, the IRM may be inconsistent.

\paragraph{Model Selection}\label{sec:conn-other-meth}


In partial Bayesian methods, a few information criteria have
been proposed.
Peixoto~\citeyearpar{Peixoto2012,Peixoto2013} used a criterion based
on the minimum description length principle.  \citet{decelle11}
proposed a BP-based framework that determines $K$ from the
Bethe free energy.

Next, for comparison with $\FFIC$, we introduce four fully Bayesian
information criteria.

\citet{daudin08} adapted the \emph{integrated classification
  likelihood (ICL)} criterion~\citep{biernacki00} to the SBM, defined
as
\begin{align}\label{eq:icl}
  \ICL = \log p(\X,\E_{\tilde p}\Z| \ML{\vTheta}(\tilde p), \ML{\veta}(\tilde p)) - \tilde\ell_N,
\end{align}
where $\tilde\ell_N$ is as defined in
Section~\ref{sec:lowerbound-ffic}. There are three main differences
with $\FFIC$: ICL 1) uses $\tilde p$ instead of $p^*$ and does not
have 2) the entropy $H$ and 3) the penalties $R_1,R_2,$ and $C$. 1)
and 2) are reasonable for a dense graph because, as discussed in
Section~\ref{sec:infer-sparse-graph}, $\tilde p$ converges to a point
estimator, which means that $p^*\to\tilde p$ and $H(\tilde p)\to 0$ at
$N\to\infty$. Also, 3) can be ignorable as a constant and hence ICL is
consistent asymptotically equivalent to $\log p(\X|K)$, if the
following strong condition holds: \assum{low-empty-prob} the
probability that the posterior generates empty (bi)clusters is
zero.\footnote{\ref{asm:low-empty-prob} is a strong version of
  \ref{asm:low-empty-prob-weak}.}  In contrast, when a graph is
sparse, $H(\tilde p)=O(N)$ and the consistency no longer holds.

\citet{latouche12} proposed a non-asymptotic counterpart of ICL that
replaces the marginal likelihood with its VB lower bound. However, the
error caused by the mean-field approximation is not asymptotically
negligible and consistency does not hold.

\citet{fujimaki12a} proposed the original formulation of FIC for
mixture models. Because the SBM is a mixture model, the FIC can be
imported into the SBM. This is defined as
\begin{align}\label{eq:FIC}
  \FIC =& \E_{p^*}[\log p(\X,\Z\mid\hat\vTheta,\hat\veta) -
  \frac{K(K+1)}{2}R_1(\bar\z) ] \notag
\\
&- \tilde\ell_N + H(p^*).
\end{align}
FIC is similar to $\FFIC$ in having the penalty term $R_1$. This
eliminates unnecessary clusters, in the same way as $\FFIC$, during
posterior inference.
However, FIC ignores the Hessian term in Laplace's method, which omits
$R_2$ from the formulation.  This makes the approximation error larger
and the regularization effect weaker than $\FFIC$ (we confirm this
empirically in the next section.)  Crucially, FIC does not take into
account the case of empty (bi)clusters in Laplace's method. Although,
like ICL, this is justifiable when \ref{asm:low-empty-prob} is
satisfied, consistency does not hold for sparse graphs. Finally, FIC
is computed by VB-based optimization; BP inference like FABBP
(Algorithm~\ref{alg:f2ab}) has not been addressed.

For ICL~\eqref{eq:icl}, if we add the entropy and move the expectation
to the outside, the criterion corresponds to the simplified version of
FIC called \emph{BICEM}~\citep{hayashi15}. We refer to this as
\emph{corrected ICL (cICL)}:
\begin{align}\label{eq:cICL}
\cICL = \E_{\tilde p}[\log p(\X,\Z| \ML{\vTheta}(\tilde p), \ML{\veta}(\tilde p))] - \tilde\ell_N + H(\tilde p).
\end{align}
Under \ref{asm:low-empty-prob}, cICL is asymptotically equivalent to
the full marginal for both dense and sparse graphs.  Nevertheless,
cICL essentially differs from $\FFIC$ in that cICL uses the
unmarginalized posterior~\eqref{eq:posterior}. Therefore, cICL does
not have an automatic model selection mechanism and the outer loop for
all model candidates is needed.

Table~\ref{tab:ic-comp} compares the above methods with $\FFIC$. It
can be seen that $\FFIC$ is the most accurate method and the
only consistent criterion for sparse graphs without any special
conditions like \ref{asm:low-empty-prob}.

Finally, we discuss a few studies that have addressed the scalability
issue of model selection. \citet{yang15} proposed a simultaneous
framework of inference and model selection by simplifying the
parameterization of the SBM. \citet{liu15} developed an FAB framework
for the factorial graph model that assumes a low-rank structure in
edge probabilities while allowing cluster overlapping. However, their
models are essentially different, and their approaches are not
applicable to the SBM.


\begin{table}[tb]
  \caption{Summary of fully Bayesian model selection on SBM. ``Accuracy'' shows asymptotic error against $p(\X)$. ``One-Pass'' indicates whether model selection is one pass or not.
    Note that ICL is consistent but its asymptotic rate is unknown.}
\label{tab:comp-ms}\label{tab:ic-comp}
\centering
{\small
\begin{tabular}{lccc}
  \hline
  Methods& \multicolumn{2}{l}{Accuracy (with/without \ref{asm:low-empty-prob})}  &One- \\ 
  & Dense& Sparse  &Pass \\ 
  \hline
  ICL~\eqref{eq:icl}   & Consistent/-- & $O(N)$/--   &  \\ 
  cICL~\eqref{eq:cICL}   &$O(1)$/--&$O(1)$/--   &  \\ 
  VB~(\citeauthor{latouche12})   &--/--&--/--    & \\ 
  FIC~\eqref{eq:FIC}     &$O(1)$/-- &$O(1)$/-- & \checkmark\\ 
  $\FFIC$~\eqref{eq:F2IC} &$O({1 \over N})$/$O(1)$ &$O({1 \over N})$/$O(1)$  & \checkmark\\ 
  \hline
\end{tabular}
}
\end{table}


\begin{figure}[tb]
  \centering
  \includegraphics[width=1\linewidth]{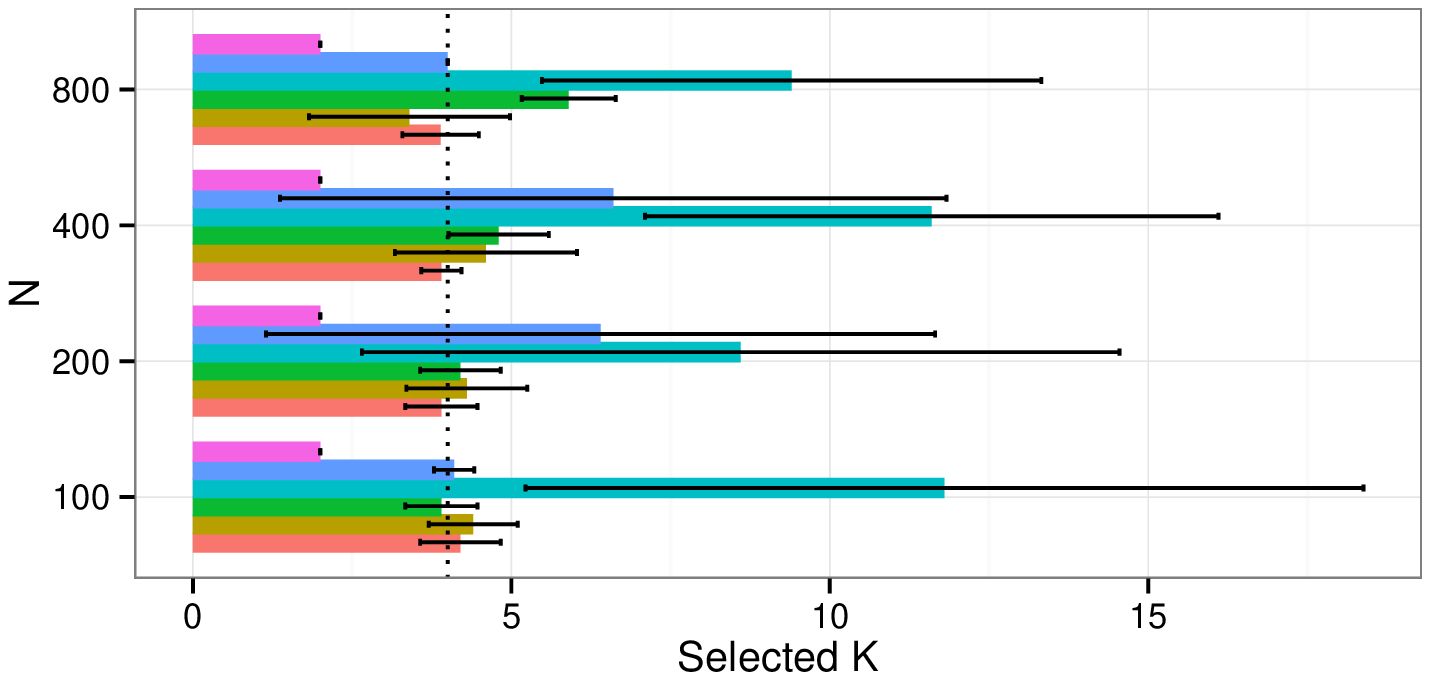}
  \includegraphics[width=1\linewidth]{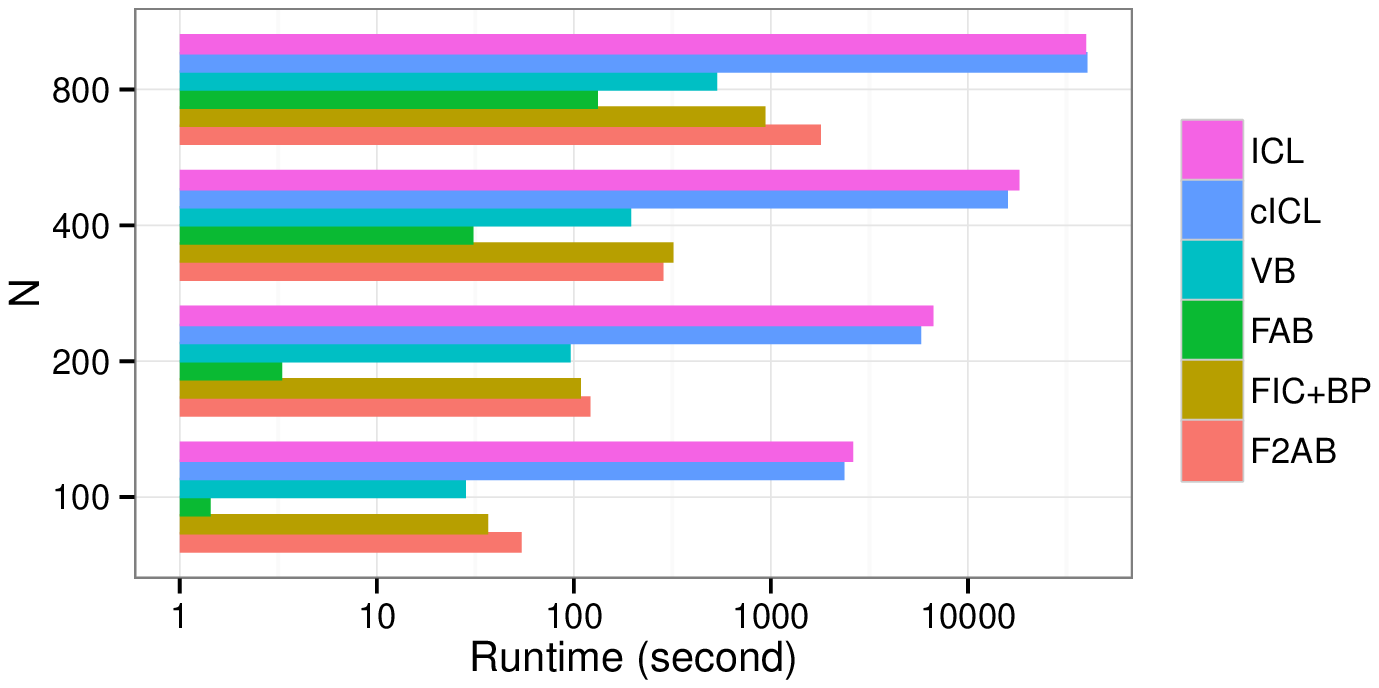}
  \caption{Synthetic data experiment: results. (Top) Means and
    standard deviations of selected number of clusters with $10$
    different random seeds. (Bottom) Means runtimes.}
  \label{fig:ex-toy}
\end{figure}

\begin{table}[ht]
\caption{Real network experiment: data.}
\label{tab:data}
\centering
{\small
\begin{tabular}{lrr}
  \hline
Data & $N$ & $M$ \\ 
  \hline
football~\citep{girvan02}     & 114 & 1224 \\ 
  euroroad~\citep{subelj11}   & 1174 & 2834 \\ 
  propro~\citep{jeong01}      & 1868 & 4406 \\ 
  netscience~\citep{newman06} & 1460 & 5484 \\ 
  email~\citep{guimera03}     & 1133 & 10902 \\ 
  names~\citep{moreno_names}  & 1773 & 18262 \\ 
  uniq~\citep{petster-friendships-hamster}  & 1858 & 25068 \\ 
  usairport~\citep{opsahl-usairport}  & 1574 & 34430 \\ 
   \hline
\end{tabular}
}
\end{table}

\begin{table*}[ht]
  \caption{Real network experiment: results. Means (and standard deviations) of selected 
    number of clusters and testing errors with $5$ different random seeds. 
    ``nNPLL'' indicates negative NPLL (smaller is better.)
    Results that were significantly better in one-sided $t$-test
    with 95\% confidence are indicated by bold font.}
\label{tab:ex-real}
\centering
{\small          
\setlength{\tabcolsep}{3pt}
\begin{tabular}{l|rrrrrrrrr}
  \hline
 & ~ & email & euroroad & football & names & netscience & propro & uniq & usairport \\ 
  \hline
\multirow{6}{*}{\rotatebox{90}{Selected $K$}}  & cICL & N/A & N/A & $7$ & N/A & N/A & N/A & N/A & N/A \\ 
   & VB & $24.8$ & $5.1$ & $13$ & $40.9$ & $21.33$ & $28.8$ & $39.1$ & $34.6$ \\ 
   & IRM & $319.99$ & $220.02$ & $27.77$ & $406.18$ & $340.95$ & $353.69$ & $497.07$ & $154.11$ \\ 
   & FAB & $13.1$ & $7.8$ & $10.9$ & $12.5$ & $16.5$ & $3.5$ & $16$ & $13.2$ \\ 
   & FIC+BP & $25.11$ & $11.44$ & $11.3$ & $21.5$ & $18.5$ & $10$ & $16.86$ & $1.88$ \\ 
   & \ttFFAB & $\mathbf{7.25}$ & $4.43$ & $\mathbf{5.75}$ & $\mathbf{6}$ & $\mathbf{6.62}$ & $\mathbf{2.62}$ & $\mathbf{8.71}$ & $2$ \\ 
  \hline
\multirow{6}{*}{\rotatebox{90}{nNPLL$\times 10^2$}} & cICL & N/A & N/A & $14.49\pm 0.75$ & N/A & N/A & N/A & N/A & N/A \\ 
   & VB & $7.28\pm 0.58$ & $4.27\pm 1.15$ & $30.57\pm 0.79$ & $6.89\pm 0.66$ & $3.96\pm 0.67$ & $5.69\pm 11.52$ & $10.97\pm 0.98$ & $17\pm 2.03$ \\ 
   & IRM & $6.93\pm 0.04$ & $2.4\pm 0.02$ & $11.3\pm 0.15$ & $5.28\pm 0.02$ & $3.65\pm 0.02$ & $2.4\pm 0.02$ & $5.82\pm 0.02$ & $\mathbf{4.81\pm 0.01}$ \\ 
   & FAB & $9.61\pm 0.55$ & $1.75\pm 0.13$ & $26.47\pm 0.53$ & $13.75\pm 1.78$ & $9.62\pm 1.67$ & $1.67\pm 0.14$ & $11.28\pm 0.41$ & $27.46\pm 2.04$ \\ 
   & FIC+BP & $3.84\pm 0.15$ & $1.65\pm 0.05$ & $9.97\pm 0.3$ & $3.16\pm 1.58$ & $\mathbf{1.09\pm 0.06}$ & $1.13\pm 0.03$ & $4.97\pm 3.39$ & $52.69\pm 15.14$ \\ 
   & \ttFFAB & $3.85\pm 0.11$ & $\mathbf{0.95\pm 0.03}$ & $\mathbf{8.99\pm 1.01}$ & $3.18\pm 0.09$ & $1.21\pm 0.11$ & $1.03\pm 0.33$ & $3.43\pm 0.11$ & $67.02\pm 12.65$ \\ 
   \hline
\end{tabular}

}
\end{table*}

\section{Experiments}\label{sec:experiments}

Following six methods were used in experiments: \texttt{ICL} and
\texttt{cICL} with BP inference, \texttt{VB} and \texttt{FAB} with the
mean-field approximation, \texttt{FIC+BP}, and
\texttt{\ttFFAB}. \texttt{FIC+BP} was the method whose objective is the
original FIC~\eqref{eq:FIC} yet the inference was done by FABBP. All
of these were implemented in Python and $\E\Z$ and
$\vPi$ were initialized using the spectral method~\citep{rohe11}.  The
model candidates of \texttt{ICL}, \texttt{cICL}, and \texttt{VB} were set to
$\{1,\dots,K_{\max}\}$.  All the hyperparameters of \texttt{VB} were
set to $1/2$ as suggested by \citet{latouche12}.

\paragraph{Synthetic Data}
First, we investigated whether the selected number of clusters $K$
coincided with one planted using synthetic data. We set $K=4$ as the
planted value, and true $\vPi$ as $\pi_{kl}=1/N$ for $k\not=l$ and
$\pi_{kk}=20/N$ as a sparse graph. We then generated
data with $\vgamma=\frac{1}{4}(1,1,1,1)$, i.e. all of the
clusters were the same size. We set $K_{\max}=20$.
%
The results (Figure~\ref{fig:ex-toy}) show clearly that \texttt{FIC+BP} and
\texttt{\ttFFAB} outperformed the other methods.  
ICL consistently underestimated $K$, as noted in
Section~\ref{sec:conn-other-meth}.
The performances of \texttt{cICL} and \texttt{VB} were unstable; they
detected $K$ correctly in most cases but produced a few very
inaccurate estimations. While the performances of \texttt{FIC+BP} and
\texttt{\ttFFAB} were similar, \texttt{\ttFFAB} provided more accurate
and stable detection, especially when $N$ was small.

\paragraph{Real Networks}
We also investigated the performance using eight real
networks~(Table~\ref{tab:data}).
Instead of \texttt{ICL}, we added the \texttt{IRM} with collapsed Gibbs
sampling~\citep{liu94}. For \texttt{IRM}, we used the same
hyperparameter setting as \texttt{VB}.  We set
$K_{\max}=\max\{20,\sqrt{N}\}$.  To measure the generalization error,
we randomly masked 1\% of all edges as missing and these were not
used in the training (during the training, these missings were imputed by
each algorithm.) After model selection, we evaluated the
normalized predictive log-likelihood (NPLL), which is the PLL divided
by $N(N+1)/2$, for those missing values.
The results in Table~\ref{tab:ex-real} show that \texttt{cICL} exceeded 48
hours for most data sets, whose results are not shown.
In terms of prediction, \texttt{FIC+BP} and \texttt{\ttFFAB} were
significantly better than the others in five data sets. In addition,
\texttt{\ttFFAB} selected the smallest $K$ for all except
``usaport'' data set. 

\paragraph{Discussion}

In the real data experiment, the difference among \texttt{FAB},
\texttt{FIC+BP}, and \texttt{\ttFFAB} highlights the significance of
our contributions in FABBP and $\FFIC$. As shown, \texttt{FIC+BP}
outperformed \texttt{FAB} for the seven data sets in
prediction. Because their objective function was the same, the
outperformance was attributed to the BP inference. Also, yet the
prediction performance was equivalent, \texttt{\ttFFAB} selected 2--4
times smaller $K$ than \texttt{FIC+BP}. In this case, because the
inference methods were the same, the difference had to come from the
difference of the objective functions, or more specifically, the
penalty term $R_2$. This supports the distinctiveness of $R_2$
discussed in Section~\ref{sec:regularizers-effect}.

Selecting a parsimonious model is a fascinating nature of our approach
that fits the principle of Occam's Razor.  If $K$ is too small
(e.g. $K=1$), the model cannot describe data well, and the
generalization error will be increased. In contrast, if $K$ is too
large (e.g. $K=O(100)$), the generalization error can be small but
interpreting its result is difficult.
As shown in the above experiments, our method resolved this trade-off
in the most successful way. Indeed, \texttt{\ttFFAB} achieved the best
prediction performance with the smallest $K$ in most of the real data
sets.

A major theoretical limitation of the $\FFAB$ algorithm is the lack of
consistency.  Although $\FFIC$ is consistent, the $\FFAB$ algorithm
does not have such guarantee due to the use of BP and the lower bound.
Nevertheless, the $\FFAB$ algorithm empirically selected better models
than the other methods. This is plausibly because of the following two
reasons.
First, because the number of message loops is smaller in a sparse
graph, FABBP might closely converge to the true marginal posterior.
Second, the $\FFAB$ algorithm started the inference from $K_{\max}$,
which was usually very large. This might expand a possible search space
and avoid getting stuck in local maxima.

%



\subsection*{Acknowledgments}
KH was supported by MEXT KAKENHI 15K16055. TK was supported by JSPS KAKENHI 26011023.

\bibliography{fabsbm}
\bibliographystyle{icml2016}

\onecolumn
\appendix
\section{Proof of Theorem~\ref{lem:joint}}

We first derive the Hessian matrix of the log-likelihood.
\begin{proposition}\label{prop:hessian}
  The Hessian matrix of the negative maximum log-likelihood $-\log
  p(\X,\Z\mid\hat\vTheta,\hat\veta)$ is given as a block diagonal
  matrix $\F= (\begin{smallmatrix}
    \F_{\vTheta} & \0\\
    \0&\F_{\veta}
  \end{smallmatrix}) $. The submatrix $\F_{\vTheta}$ is diagonal
  having $K(K+1)/2$ elements, where each element corresponds to the
  second derivative with respect to $\theta_{kl}$ for $k=1,\dots,K$
  and $k\leq l \leq K$. Its diagonal element is given as
\begin{align}\label{eq:F_theta}
  \bar M_{kl} \hat\pi_{kk}(1-\hat\pi_{kk})
\end{align}
where $\bar M_{kl}$ is defined in Theorem~\ref{lem:joint}.
Another submatrix $\F_{\veta}$ is given as
\begin{align}\label{eq:F_eta}
\F_{\veta}=N(\diag(\hat\vgamma_{\quot K})-\vgamma_{\quot
  K}\hat\vgamma_{\quot K}^\T),
\end{align}
where $\vgamma_{\quot K}=(\gamma_1,\dots,\gamma_{K-1})^\T$.
\end{proposition}
\begin{proof}
  Since $\vTheta$ and $\veta$ have no interaction in $\log
  p(\X,\Z|\vTheta,\veta)$, the off-diagonal elements of $\F$ are
  zero. Now we check the Hessian w.r.t. $\vTheta$, which is
  \begin{align}
    &\frapp{}{\theta_{kl}}\{-{N+1 \choose 2}^{-1}\log p(\X,\Z\mid\hat\vTheta,\hat\veta)\}
\\
&=\frac{2}{N(N+1)}\sum_{i\leq j}z_{ik}z_{jl}\psi''(\hat\theta_{kl})
\\
&=\frac{1}{N(N+1)}(\sum_{i, j}z_{ik}z_{jl}\psi''(\hat\theta_{kl}) + \sum_{i}z_{ik}z_{il}\psi''(\hat\theta_{kl}))
\\
&=\frac{N}{N+1}(\bar z_{k}\bar z_{l} + \frac{1}{N^2}\sum_{i}z_{ik}z_{il})\psi''(\hat\theta_{kl}).
\end{align}
Since $\z_i$ is 1-of-$K$-coded, $\sum_{i}z_{ik}z_{il}=0$ for $k\not=l$
and $\sum_{i}z_{ik}z_{ik}=\sum_{i}z_{ik}=N\bar z_k$. Also, since
$\psi'(\cdot)$ is the sigmoid function, its derivatives is written as
\begin{align}
  \psi''(\hat\theta_{kl})
  &=\psi'(\hat\theta_{kl})(1-\psi'(\hat\theta_{kl}))
  \\
  &=\hat\pi_{kl}(1-\hat\pi_{kl})
\end{align}
By substituting these, we obtain Eq.~\eqref{eq:F_theta}. For $\veta$,
  \begin{align}
    &\nabla_{\veta}\nabla_{\veta}\frac{-\log p(\X,\Z\mid\hat\vTheta,\hat\veta)}{N}
=\nabla_{\veta}\nabla_{\veta}\phi(\hat\veta)
  \end{align}
and
\begin{align}
  \frap{\phi(\hat\veta)}{\eta_{k}}&=\phi'_k(\hat\veta)\equiv\frac{\mathrm{e}^{\hat\eta_k}}{1+\sum_{p<K}\mathrm{e}^{\hat\eta_p}},
\\
\frap{^2\phi(\hat\veta)}{\eta_{k}\partial\eta_l}
&=
\frap{\phi'_k(\hat\veta)}{\eta_l}
\\
&=
\ind(k=l)\frac{\mathrm{e}^{\hat\eta_k}}{1+\sum_{p<K}\mathrm{e}^{\hat\eta_p}}
-\frac{\mathrm{e}^{\hat\eta_k}\mathrm{e}^{\hat\eta_l}}{(1+\sum_{p<K}\mathrm{e}^{\hat\eta_p})^2}
\\
&= 
\ind(k=l)\phi'_k(\hat\veta)-\phi'_k(\hat\veta)\phi'_l(\hat\veta)
\\
&= 
\ind(k=l)\hat\gamma_{k}-\hat\gamma_{k}\hat\gamma_{l}.
\end{align}
This yields Eq.~\eqref{eq:F_eta}.
\end{proof}

We then consider the joint marginal. We see that the marginalization
is divided into into two parts:
\begin{align}
  \log p(\X,\Z)
  =&
  \log \int p(\X,\Z\mid\vTheta,\veta)p(\vTheta)p(\veta)\d\vTheta\d\veta
\\
  =&
  \log \int p(\X\mid\Z,\vTheta)p(\vTheta)\d\vTheta +
  \log \int p(\Z\mid\veta)p(\veta)\d\veta.\label{eq:joint-decomposed}
\end{align}
The first term can further be broken down into the marginals with
respect to $\{\theta_{kl}\}$, which is evaluated by the next lemma.
\begin{lemma}\label{lem:deg-mar-theta}
Let $\vTheta_{\quot kl}=\{\theta_{ij}|i\not=k \wedge j\not=l \}$. Then,
  \begin{align}
    &\log \int p(\X|\z_{k}, \z_l, \theta_{kl}) p(\vTheta)\d\theta_{kl} 
    =
    \log p(\vTheta_{\quot kl}) 
\\
    +&
    \begin{cases}
      \log p(\X|\z_{k}, \z_l, \ML{\theta}_{kl}) + \log p(\ML{\theta}_{kl}|\vTheta_{\quot kl})
      -\frac{1}{2}\big(\log \ML{\theta}_{kl}(1-\ML{\theta}_{kl})
      +\log \frac{\bar M_{kl}}{2\pi}
      + O(1/N^2\bar z_k\bar z_l)
      & \text{if $\zz_{kl},\bar z_k,\bar z_l>0$}
\\
     P_{kl}
     &\text{if $\zz_{kl}=0$ and $\bar z_k,\bar z_l>0$}
\\
     0
     &\text{if $\bar z_k=0$ or $\bar z_l=0$}
    \end{cases}
  \end{align}
  $P_{kl}$ and $\bar M_{kl}$ are defined in Theorem~\ref{lem:joint}.
\end{lemma}
\begin{proof}

For the integral, there are three cases we have to consider.

\textbf{Case 1: $\zz_{kl},\bar z_k,\bar z_l>0$}\\
In this case, $-\infty<\ML{\theta}_{kl}<\infty$ and $\psi''(\ML{\theta}_{kl})>0$,
meaning the conditions for Laplace's method are satisfied.  We then
use Laplace's method and obtain the result.

\textbf{Case 2: $\zz_{kl}=0$ and $\bar z_k,\bar z_l>0$}\\
In this case, the maximum occurs at the endpoint
$\ML{\theta}_{kl}\to-\infty$, and we cannot use Laplace's method. We
then leave it as an exact expression of the integral, which is
$P_{kl}$.

\textbf{Case 3: $\bar z_k=0$ or $\bar z_l=0$}\\
In this case, $p(\X|\z_{k}, \z_l, \theta_{kl})=1$ and the integral
boils down the marginalization of the prior $\int
p(\vTheta)\d\theta_{kl}=p(\vTheta_{\quot kl})$.
\end{proof}

The second term of \eqref{eq:joint-decomposed} is evaluated as the next lemma.
\begin{lemma}\label{lem:deg-mar-eta}
  Let $\S=\{k|\bar z_k>0\}\quot K$. Then,
  \begin{align}
    \log \int p(\Z|\veta) p(\veta)\d\veta
    \approx&
    N\sum_{k\in \S}\bar z_k\ML{\eta}_k - N\phi(\ML{\veta}_{\S})
    + \log p(\ML{\veta}_{\S}) 
    - \frac{1}{2}\sum_{k\in\S}\log \bar z_k - \frac{|S|}{2}\log \frac{N}{2\pi}
    +Q_{\S}(N)
  \end{align}
  where $Q_{\S}$ and $\ML{\alpha}$ are defined in
  Theorem~\ref{lem:joint}.
\end{lemma}
\begin{proof}
  By denoting $\alpha = 1+\sum_{k\in\S}\e^{\eta_k}$, $\beta=1 +
  \sum_{l\notin\S}\e^{\eta_l-\log\alpha}$ and using the relation
  ${\e^x\over a}=\e^{x-\log a}$, we have
  \begin{align}
    \log \int p(\Z|\veta) p(\veta)\d\veta 
    =& \log\int \exp(N\sum_{k\in \S}\bar
    z_k\ML{\eta}_k) \left(\frac{1}{1+\sum_{k\in\S}\e^{\ML{\eta}_k} +
        \sum_{l\notin\S}\e^{\eta_l}}\right)^N p(\veta)
\d\veta
\\
    =& \log\int \exp(N\sum_{k\in \S}\bar
    z_k\ML{\eta}_k) \left(\frac{1}{\alpha\beta}\right)^N p(\veta)
\d\veta
  \end{align}
  By using change of variable $\xi_l=\eta_l - \log\alpha$ for
  $l\notin\S$, we can rewrite this as
  \begin{align}\label{eq:confusing-marginal-eta}
    \log\int \exp(N\sum_{k\in \S}\bar
    z_k\ML{\eta}_k) \left(\frac{1}{\alpha}\right)^N\left(\frac{1}{1+\sum_{l\notin\S}\e^{\xi_l}}\right)^N p(\veta_{\S},\vxi_{\quot \S} + \log\alpha)
    \d\veta_{\S}\d\vxi_{\quot \S}.
  \end{align}
  In \eqref{eq:confusing-marginal-eta}, the marginal
  w.r.t. $\veta_{\S}$ approximated by Laplace's method as
\begin{align}
  \mathcal{L} \equiv N\sum_{k\in \S}\bar z_k\ML{\eta}_k - N\phi(\ML{\veta}_{\S})
    + \log p(\ML{\veta}_{\S}) 
    - \frac{1}{2}\log|\F_{\S}| - \frac{|S|}{2}\log \frac{N}{2\pi}
\end{align}
where
\begin{align}
  \log|\F_{\S}| =& \log|\nabla_{\ML{\veta}_{\S}}\nabla_{\ML{\veta}_{\S}}\phi(\veta)|
\\
  =& \log(1 - \sum_{s\in\S}\ML{\gamma}_{s}) + \log\sum_{s\in\S}\ML{\gamma}_{s}
\\
  =& \log\ML{\gamma}_{K} + \sum_{s\in\S}\ML{\gamma}_{s}
\\
  =& \sum_{s\in\S\cup K}\bar{\z}_{s}.
\end{align}
Then, \eqref{eq:confusing-marginal-eta} is rewritten as
\begin{align}
  \eqref{eq:confusing-marginal-eta}
  \approx&
  \mathcal{L}
  +\log\int \left(\frac{1}{1+\sum_{l\notin\S}\e^{\xi_l}}\right)^N p(\vxi_{\quot \S} + \log\ML{\alpha}|\ML{\veta}_{\S})
  \d\vxi_{\quot \S} 
  \\
  \approx&
  \mathcal{L}+Q_{\S}(N).
  \end{align}
\end{proof}

Combining Lemmas~\ref{lem:deg-mar-theta} and \ref{lem:deg-mar-eta}
gives \eqref{eq:asymptotic-joint}.



\end{document}